\documentclass[conf]{new-aiaa}

\usepackage[utf8]{inputenc}

\usepackage{acronym}

\usepackage{blindtext} 

\usepackage[T1]{fontenc}

\usepackage{amsmath}
\usepackage{amssymb}
\usepackage{amsfonts}
\usepackage{graphicx}
\usepackage{textcomp}
\usepackage{setspace}
\usepackage{booktabs}
\usepackage{multirow}
\usepackage[table]{xcolor} 
\usepackage{epsfig}
\usepackage{svg}
\usepackage[ruled,vlined,linesnumbered]{algorithm2e} 
\usepackage{bm}
\usepackage{tabularx}
\usepackage{placeins} 
\usepackage{siunitx} 

\usepackage{xurl}

\newcommand{\naturals}{\mathbb{N}}

\newcommand{\reals}{\mathbb{R}}
\newcommand{\R}{\reals}
\newcommand{\Rnonneg}{\reals_{\geq 0}}

\renewcommand{\emptyset}{\varnothing}

\newcommand{\norm}[1]{\left\Vert #1 \right \Vert}

\newcommand{\Ccal}{\mathcal{C}}

\newcommand{\Gcal}{\mathcal{G}}

\newcommand{\Pcal}{\mathcal{P}}

\newcommand{\Scal}{\mathcal{S}}

\newcommand{\Ucal}{\mathcal{U}}

\newcommand{\Xcal}{\mathcal{X}}

\newcommand{\eqn}[1]{\begin{align} #1 \end{align}}

\newcommand{\seqn}[2][]{
\begin{subequations}
#1
\begin{align} #2 \end{align}
\end{subequations}
}

\newcommand{\gatekeeper}{\texttt{gatekeeper}}

\newcommand{\nom}{{\rm nom}}
\newcommand{\can}{{\rm can}}
\newcommand{\com}{{\rm com}}
\newcommand{\back}{{\rm bak}}

\newcommand{\rrts}{\textit{RRT*}}

\newcolumntype{g}{>{\columncolor{gray!30}}r}

\acrodef{BCH}[BCH]{Baker-Campbell-Hausdorff}
\acrodef{CBF}[CBF]{Control Barrier Function}
\acrodef{CBF-QP}[CBF-QP]{Control Barrier Function Quadratic Program}
\acrodef{CDC}[CDC]{Conference on Decision and Control}
\acrodef{CESDF}[CESDF]{Certified ESDF}
\acrodef{CLF}[CLF]{Control Lyapunov Function}
\acrodef{CVO}[C-VO]{Certified Visual Odometry}
\acrodef{DCT}[DCT]{Discrete Cosine Transform}
\acrodef{DMP}[DMP]{Distance Map Planner}
\acrodef{EKF}[EKF]{Extended Kalman Filter}
\acrodef{ESDF}[ESDF]{Euclidean Signed Distance Field}
\acrodef{EZ}[EZ]{Engagement Zone}
\acrodef{FOV}[FoV]{Field of View}
\acrodef{FPV}[FPV]{First Person View}
\acrodef{GNC}[GNC]{Graduated-Nonconvexity}
\acrodef{GP}[GP]{Gaussian Process}
\acrodef{HOCBF}[HOCBF]{Higher Order Control Barrier Function}
\acrodef{ICCBF}[ICCBF]{Input-Constrained Control Barrier Function}
\acrodef{IEEE}[IEEE]{Institute of Electrical and Electronics Engineers}
\acrodef{IMU}[IMU]{Inertial Measurement Unit}
\acrodef{ISS}[ISS]{Input-to-State}
\acrodef{KF}[KF]{Kalman Filter}
\acrodef{ML}[ML]{Machine Learning}
\acrodef{MPC}{Model Predictive Control}
\acrodef{NGPKF}[NGPKF]{Numerical Gaussian Process Kalman Filter}
\acrodef{ODE}[ODE]{Ordinary Differential Equation}
\acrodef{QP}[QP]{Quadratic Program}
\acrodef{RGBD}[RGBD]{RGB-Depth}
\acrodef{RL}[RL]{Reinforcement Learning}
\acrodef{RoS}[RoS]{Rate of Spread}
\acrodef{SDE}[SDE]{Stochastic Differential Equation}
\acrodef{SDF}[SDF]{Signed Distance Field}
\acrodef{SFC}[SFC]{Safe Flight Corridor}
\acrodef{SLAM}[SLAM]{Simultaneous Localization and Mapping}
\acrodef{SOS}[SOS]{Sum of Squares}
\acrodef{SVD}[SVD]{Singular Value Decomposition}
\acrodef{TCAC}[TCAC]{Technical Committee on Aerospace Controls}
\acrodef{TLS}[TLS]{Truncated Least Squares}
\acrodef{TSDF}[TSDF]{Truncated Signed Distance Field}
\acrodef{TSD}[TSD]{Target Spatial Distribution}
\acrodef{VIO}[VIO]{Visual Inertial Odometry}
\acrodef{VO}[VO]{Visual Odometry}
\acrodef{WLS}[WLS]{Weighted Least Squares}

\usepackage{amsthm}

\theoremstyle{plain}
\newtheorem{theorem}{Theorem}

\newtheorem{lemma}{Lemma}
\newtheorem{problem}{Problem}
\newtheorem{definition}{Definition}

\newtheorem{assumption}{Assumption}
\newtheorem{remark}{Remark}

\theoremstyle{remark}

\AtBeginEnvironment{definition}{\let\oldemph\emph \renewcommand{\emph}[1]{\textbf{#1}}}
\AtEndEnvironment{definition}{\let\emph\oldemph}

\usepackage{cleveref}
\crefformat{problem}{Problem~#2#1#3}
\crefformat{assumption}{Assumption~#2#1#3}

\usepackage{float}
\usepackage{verbatim}
\usepackage{graphicx}
\usepackage{multirow}
\usepackage{amsmath}
\usepackage[version=4]{mhchem}
\usepackage{siunitx}
\usepackage{longtable,tabularx}
\title{Multi-Agent \gatekeeper{}: Safe Flight Planning and Formation Control for Urban Air Mobility}

\author{Marshall Vielmetti\footnote{Masters Student, Department of Electrical \& Computer Engineering, 1301 Beal Ave, Ann Arbor, MI 48109}}
\affil{University of Michigan, Ann Arbor, MI, 48109, USA}
\author{Devansh R. Agrawal\footnote{Postdoc, Department of Robotics, 1320 Beal Ave, Ann Arbor, MI 48109.}}
\affil{University of Michigan, Ann Arbor, MI, 48109, USA}
\author{Dimitra Panagou\footnote{Associate Professor, Department of Robotics and Department of Aerospace Engineering, 2505 Hayward St, Ann Arbor, MI 48109).}}
\affil{University of Michigan, Ann Arbor, MI, 48109, USA}

\begin{document}

\crefname{definition}{Def.}{Defs.}
\crefname{assumption}{assumption}{assumptions}
\Crefname{assumption}{Assumption}{Assumptions}
\crefname{Assumption}{assumption}{assumptions}
\Crefname{Assumption}{Assumption}{Assumptions}

\maketitle

\begin{abstract}
We present Multi-Agent \gatekeeper{}, a framework that provides provable safety guarantees for leader-follower formation control in cluttered 3D environments.
Existing methods face a trad-off: online planners and controllers lack formal safety guarantees, while offline planners lack adaptability to changes in the number of agents or desired formation.
To address this gap, we propose a hybrid architecture where a single leader tracks a pre-computed, safe trajectory, which serves as a shared trajectory backup set for all follower agents. 
Followers execute a nominal formation-keeping tracking controller, and are guaranteed to remain safe by always possessing a known-safe backup maneuver along the leader's path. We formally prove this method ensures collision avoidance with both static obstacles and other agents.
The primary contributions are: (1) the multi-agent \gatekeeper{} algorithm, which extends our single-agent gatekeeper framework to multi-agent systems; (2) the trajectory backup set for provably safe inter-agent coordination for leader-follower formation control; and (3) the first application of the gatekeeper framework in a 3D environment. We demonstrate our approach in a simulated 3D urban environment, where it achieved a 100\% collision-avoidance success rate across 100 randomized trials, significantly outperforming baseline CBF and NMPC methods. Finally, we demonstrate the physical feasibility of the resulting trajectories on a team of quadcopters.
\end{abstract}

\section{Nomenclature}

{\renewcommand\arraystretch{1.0}
\noindent\begin{longtable*}{@{}l @{\quad=\quad} l@{}}
$\Xcal, \Xcal_\text{obs}, \Scal$ & State space, Obstacle space, Safe set\\
$\Ucal$ & Set of admissible control inputs.\\
$\Ccal$ & Backup set\\
$\psi, \gamma, \omega \in \R$ & Heading angle (rad), Climb angle (rad), Turning Rate (rad/s)\\
$\delta \in \Rnonneg$ & Minimum inter-agent separation distance\\
$\epsilon \in \Rnonneg$ & Separation buffer to account for curvature\\
$d_i^* \in \R^d$ & Desired offset of follower agent $i$ from leader\\
$r_i \in \R^d$ & Position component of state of agent $i$.\\
$N \in \naturals$ & Number of agents\\
$\pi^\nom, \pi^\back$ & Nominal controller, Backup controller\\
$p_{k,i}^\nom, u_{k_i,i}^\nom $ & $k$-th Nominal Trajectory of agent $i$.\\
$p_{k,i}^\can, u_{k_i,i}^\nom$ & $k$-th Candidate Trajectory of agent $i$.\\
$p_{k,i}^\com, u_{k_i,i}^\nom$ & $k$-th Committed Trajectory of agent $i$.\\
$p_L, u_L$ & Precomputed leader's trajectory\\ 
$T_H, T_B \in \Rnonneg$ & Planning horizon, backup time horizon\\
\end{longtable*}}

\section{Introduction}\label{section:introduction}
Multi-agent formation control has a wide variety of applications, including
surveillance \cite{xiaMultiAgentReinforcementLearning2022}, exploration \cite{joFoXFormationAwareExploration2024}, and search and rescue \cite{afraziDensityDrivenFormationControl2025}.
However, providing rigorous safety guarantees for complex multi-agent systems operating in real-world environments remains a significant challenge in deploying robotics in sensitive applications. 

Methods for multi-agent formation control can be broadly classified into three categories: (1) leader-follower, (2) virtual structures, and (3) behavior-based methods \cite{leeDecentralizedBehaviorbasedFormation2018,barfootMotionPlanningFormations2004}.
Leader-follower methods designate some agents as "leaders" and the rest as "followers". Leader agents execute some desired trajectory, while followers attempt to maintain some fixed offset from the leader \cite{roldaoLeaderfollowingTrajectoryGenerator2014,xiaoLeaderFollowerConsensusMultiRobot2019}.
These methods are appealing due to their simplicity, but often lack rigorous safety guarantees, as followers may collide with obstacles or each other while attempting to maintain formation \cite{liuSurveyFormationControl2018,wangVisionBasedFlexibleLeader2020}.
Virtual-structure based methods treat the entire formation as a single entity,
and attempt to maintain a rigid formation by controlling the motion of this virtual
structure \cite{kar-hantanVirtualStructuresHighprecision1996,askariUAVFormationControl2015,zhouAgileCoordinationAssistive2018}. These methods
can provide strong formation tracking guarantees, but often struggle to adapt to changes in the number of agents or desired formation, and are difficult to scale to large numbers of agents \cite{OH2015424}.
Behavior-based methods define a set of behaviors for each agent, such as
obstacle avoidance, formation keeping, and goal seeking, and combine these behaviors to generate control inputs \cite{balchBehaviorbasedFormationControl1998}.
These methods are often robust to changes in the environment and number of agents, but can be difficult to tune and may not provide formal safety guarantees \cite{leeDecentralizedBehaviorbasedFormation2018}.
In some cases, it is desirable to pre-plan safe, dynamically feasible trajectories offline, such that they may be executed by agents without having to worry about achieving real time performance. 
Methods like multi-agent RRT* \cite{capMultiagentRRTSamplingbased2013} and other planning techniques \cite{barfootMotionPlanningFormations2004,zhouHybridPathPlanning2022} can be used to generate safe trajectories offline, but these methods fail to scale to large numbers of agents, or high-dimensional state spaces.
Furthermore, if the desired formation or number of agents change, these offline solutions become unusable.

Thus, there is a need for methods that can provide rigorous safety guarantees, while being adaptable to changes in the number of agents or desired formation, and computationally efficient enough to run online in real time.
This gap motivates our approach, multi-agent \gatekeeper{}.
We propose a hybrid online-offline leader-follower approach, where a single path is computed 
offline for a designated leader-agent, and followers attempt to maintain a formation around the leader while deviating to ensure safety.
Our approach provides rigorous theoretical safety guarantees, while being adaptable to changes in the number of agents or desired formation, and is computationally efficient enough to run online in real time.
This paper extends our lab's previous work in online safety verification and control for single-agents, \gatekeeper{} \cite{agrawalGatekeeperOnlineSafety2024}, which we extend to multi-agent systems for the first time, in the context of a distributed, leader-follower formation control problem. Our contributions are as follows:
\begin{enumerate}
    \item We present multi-agent \gatekeeper{}, an adaptation of our single-agent safety framework \gatekeeper{} to multi-agent systems to achieve formal safety guarantees. 
    \item We present the first application of \gatekeeper{} in a 3D environment, demonstrating our approach on a 3D Dubins aircraft model in a dense, urban-like environment.
    \item We apply our method to a leader-follower formation control problem in both simulation and on hardware, demonstrating the ability to maintain formation while ensuring safety, and showing our method achieves safety where baselines fail.
\end{enumerate}

This paper is organized as follows: \Cref{section:preliminaries} introduces the preliminaries and problem statement, \cref{section:proposed_solution} details our proposed solution, \cref{section:simulation_results} describes our simulated experiments and baseline comparisons, \cref{section:hardware_experiments} describes our hardware experiments, and finally conclusions and future directions are discussed in \cref{section:conclusion}.

\section{Preliminaries \& Problem Statement}\label{section:preliminaries}
We will first introduce some preliminaries and notation, then formally state the problem we are trying to solve.
\vspace{-0.5cm}
\subsection{Preliminaries}\label{sec:subsection:preliminaries}

Consider a nonlinear system,
\begin{equation}\label{eq:dynamical_system}
    \dot x = f(x, u)
\end{equation}

where $x \in \Xcal \subseteq \R^n$ is the state and $u \in \Ucal \subseteq \R^m$ is the control input. $f: \Xcal \times \Ucal \to \R^n$ is piecewise continuous and locally Lipschitz in $x$ and $u$.  Let $r \in \R^d$, $d \in \{2, 3\}$ denote the position component of the state $x$.

\begin{definition}[Agent]\label{def:agent}
    An \textbf{agent} $i \in V$ represents a robotic system, which satisfies the following:
    \begin{enumerate}
        \item The agent's state $x_i$ evolves through a system satisfying \eqref{eq:dynamical_system}.
        \item Collision radius $\delta \in \R_{>0}$. Agents must maintain a minimum separation of $\delta$.
        \item Agents can communicate over some undirected graph $\Gcal$ as in \cref{def:undirected_graph}.
    \end{enumerate} 
    Agents can be classified as \emph{leaders} or \emph{followers}.  Leaders execute some desired trajectory, while followers attempt to maintain some fixed offset from the leader.
\end{definition}

Agents attempt to navigate through a known environment, which contains some set of static obstacles $\Xcal_\text{obs} \subset \Xcal$. 
The safe set is defined as $\Scal = \Xcal \setminus \Xcal_\text{obs}$.

\begin{definition}[Trajectory]
    A \textbf{trajectory} is the tuple $(T, p, u)$, where $T \subseteq \R_{\geq 0}$ is a time interval, $p: T \to \Xcal$ is the state trajectory, and $u: T \to \Ucal$ is a control trajectory, which together satisfy the system dynamics,
    \eqn{
        \dot{p}(t) = f(p(t), u(t)), \forall t \in T.
    }
\end{definition}
The position component of the trajectory is denoted $r(t)$, where $r(t)$ is the position component of $p(t)$.

\begin{definition}[Arc Length]\label{def:arc_length}
    Given a trajectory $(T, p, u)$ the \emph{arc length} between times $t_1, t_2 \in T$ is defined as,
    \eqn{
        d(t_1, t_2) &= \int_{t_1}^{t_2} \norm{\dot{r}(\tau)}d\tau
    }
\end{definition}

\begin{definition}[Undirected Graph]\label{def:undirected_graph}
   Let $\Gcal = (V, E)$ be an undirected graph, where $V$ is the set of vertices and $E \subseteq V \times V$ is the set of edges. Each vertex $i \in V$ represents an agent, and an edge $(i, j) \in E$ indicates that agents $i$ and $j$ can communicate with each other.
\end{definition}

\begin{assumption}[Communications]\label{assumption:communications}
    We assume $\Gcal$ is fully connected, and has no communication delay or bandwidth limitations.
\end{assumption}
\begin{remark}
    \Cref{assumption:communications} is adopted to simplify the safety proof. Future work will investigate relaxing this assumption to allow for time-varying graphs, communication delays, and bandwidth constraints.
\end{remark}


\subsection{\gatekeeper{} Preliminaries}
We now introduce definitions and assumptions from \gatekeeper{} \cite{agrawalGatekeeperOnlineSafety2024,agrawalOnlineSafetyMultiple2025}, which forms the basis of our approach.

\begin{definition}[Backup Set]\label{def:backup_set}
    A \emph{backup set}, denoted $\Ccal \subseteq \Scal$ is some subset of the safe set, defined with some feedback controller $\pi^\back: \Xcal \to \Ucal$ (the \emph{backup controller}), such that  the backup controller renders the backup set control-invariant. That is, the closed-loop system $\dot{x} = f(x,\pi^\back(x))$ satisfies
    \eqn{
        x(t_k) \in \Ccal \implies x(t) \in \Ccal, \forall t \geq t_k.
    }
\end{definition}

\begin{definition}[Nominal Trajectory]
From some state $x(t_k) \in \Xcal$ at time $t_k$, a \textbf{nominal trajectory} is the tuple $([t_k, t_k + T_H], p_k^\nom, u_k^\nom)$, where $T_H \in \Rnonneg$ is the planning horizon. 
\end{definition}
Nominal trajectories are often generated by online planners, and may violate safety constraints, or terminate in some unrecoverable state. 

\begin{definition}[Backup Trajectory]\label{def:backup_trajectory}
    Let $T_B \in \Rnonneg$ be finite. For any $t_s \in \Rnonneg$, and $x_s \in \Xcal$, a trajectory $([t_s, \infty), p_k^\back, u_k^\back)$ is a \emph{backup trajectory} if:
    \eqn{
        p_k^\back(t_s + T_B) \in \Ccal,
    }
    and for all $t \geq t_s + T_B$
    \eqn{
        u^\back(t) = \pi^\back(p^\back(t)),
    }
    where $\Ccal \subset S$ is a backup set as in \cref{def:backup_set}, and $\pi^B$ is a backup controller that renders $\Ccal$ control-invariant. 
\end{definition}

\begin{definition}[Candidate Trajectory]\label{def:candidate_trajectory}
    Given a state $x(t_k) \in \Xcal$ at some time $t_k \in \Rnonneg$, a \emph{candidate trajectory} consists of a nominal trajectory, switch time, and backup trajectory. A candidate with switch time $t_s$ is denoted $(p^{\can, t_s}, u^{\can, t_s})$ and defined over $[t_k, \infty)$. The candidate is defined as a piecewise trajectory,
    \eqn{
        p^{\can, t_s}(t), u^{\can, t_s}(t) = 
        \begin{cases}
            p^{\nom} (t), u^\nom (t), & \quad t \in [t_k, t_s)\\
            p^\back (t), u^\back (t), &\quad t \in [t_s, \infty).
        \end{cases}
    }
\end{definition}
A candidate controller is thus defined by an agent executing its nominal trajectory until it reaches the switch time, then a backup controller for all future time.


\subsection{Problem Statement}
We now introduce a formation control problem, where agents attempt to maintain some fixed offset from a designated leader, while 
avoiding obstacles and each other.

\begin{definition}[Formation Control]
Leader-follower \textbf{formation control} is defined by designating one agent as a leader, and the rest as followers. Each follower agent $i$ then attempts to maintain some fixed offset $d_i^*$ from the leader, while the leader executes its desired trajectory.
\end{definition}

\begin{definition}[Formation Error]
Given a formation control problem with $i=1\dots N-1$ follower agents attempting to maintain some fixed displacement $d_i^* \in \R^d$ relative to a designated leader agent, we can write the \textbf{formation error} at time $t_k$ as
\begin{equation}
    E(t_k) = \sum\limits_{i=1}^{N-1} \norm{(r_{L}(t_k) - r_{i}(t_k)) - d_i^*}^2,
\end{equation}
where $r_{L}(t_k)$ is the position of the leader agent at time $t_k$, $r_{i}(t_k)$ is the position of the follower agent, and $d_i^*$ is the desired displacement of the follower agent.
\end{definition}
We will seek to minimize formation error of follower agents, but prioritize deviations from the nominal formation keeping controller to maintain safety.

\begin{problem}[Formation Control Problem]\label{problem:multi_agent_formation_control}
Consider a set $V$ of $N$ homogeneous agents defined as in \cref{def:agent}, consisting of one designated leader and $N-1$ followers.
Agents navigate a known environment, and must remain within the safe set $\Scal$ as defined in \cref{sec:subsection:preliminaries}, and avoid inter-agent collisions.
Agents communicate over an undirected graph $\Gcal = (V, E)$ satisfying \cref{assumption:communications}.

The leader agent executes a pre-computed, safe, dynamically feasible trajectory $([t_0, t_f], p_L, u_L)$ while each follower $i$ seeks to track a fixed offset $d_i^*$ relative to the leader.
We aim to minimize the total formation error $E(t)$ over the mission horizon $[t_0, t_f]$, while ensuring safety constraints are satisfied. Formally, the global formation control problem is stated as,
\seqn{
    \min & \quad \int_{t_0}^{t_f} E(t) \; dt\\
    \text{s.t.} & \quad x_i(t) \in \mathcal{S}, \quad i=1,\ldots,N, \; t \in [t_0, t_f]\\
    & \quad \norm{r_i(t) - r_j(t)} \geq \delta, \quad \forall i \neq j, \; t \in [t_0, t_f]\\
    & \quad \dot x_i(t) = f(x_i(t), u_i(t)), \quad i=1,\ldots,N, \; t \in [t_0, t_f]
}
\end{problem}
\vspace{-0.5cm}
\section{Proposed Solution}\label{section:proposed_solution}
We propose multi-agent \gatekeeper{}, an extension of our single-agent safety framework \gatekeeper{} \cite{agrawalGatekeeperOnlineSafety2024} to multi-agent systems. While the original framework provides formal safety guarantees in for single agents, even in the presence of dynamic obstacles, this extension addresses the challenge of avoiding other agents. 

Our key insight is that inter-agent communication allows agents to treat other agents as predictable, time-varying obstacles. By sharing committed trajectories, agents can certify candidates as safe in a distributed manner. 

We first present the general multi-agent \gatekeeper{} algorithm in \cref{sec:subsection:multi_agent_gatekeeper}. We then apply this algorithm to the leader-follower formation control \cref{problem:multi_agent_formation_control} in \cref{sec:subsection:formation_control_with_multi_agent_gatekeeper}, leveraging the leader's path as a shared backup maneuver to provide formal proofs of safety in \cref{sec:subsection:backup_trajectory_generation_and_coordination}.

\subsection{Multi-Agent gatekeeper{}}\label{sec:subsection:multi_agent_gatekeeper}
Each follower agent $i$ executes multi-agent \gatekeeper{} independently (\cref{alg:multi_agent_gatekeeper}) at discrete iterations $k_i \in \naturals$, which occur at times $t_{k_i}$, such that $t_{k_i} < t_{k_i + 1}$. Our method is \emph{computationally distributed and asynchronous} (agents plan locally) but \emph{informationally centralized} (agents require the committed trajectories of all other agents).

At every iteration $k_i$, agent $i$ performs the following:
\begin{enumerate}
    \item \textbf{Receive} current committed trajectories of all other agents, $\Pcal^\com$.
    \item \textbf{Construct} a goal-oriented nominal trajectory $([t_{k_i,i}, t_{k_i,i} + T_H], p_{k_i,i}^\nom, u_{k_i,i}^\nom)$ over the horizon $T_H$.
    \item \textbf{Attempt} to construct a candidate trajectory $([t_{k_i,i}, \infty), p_{k_i,i}^\can, u_{k_i,i}^\can)$ as in \cref{def:candidate_trajectory} using \cref{alg:construct_candidate}.
    \item \textbf{Commit} to the candidate trajectory if valid (\cref{def:valid}), otherwise, continue the previous committed trajectory.
\end{enumerate}
\begin{algorithm}[H]
    \caption{Multi-Agent \gatekeeper{} ( Agent $i$ at time $t_{k_i}$)}
    \label{alg:multi_agent_gatekeeper}
    \DontPrintSemicolon
    
    $\Pcal_\com \leftarrow p_{k_j,j}^\com, \; \forall j \neq i$    \tcp*[R]{Receive other committed trajectories.}

    $([t_{k_i,i}, t_{k_i,i} + T_H], p_i^\nom, u_i^\nom) \leftarrow (x_i^\nom(t), \pi_i^\nom(x_i^\nom(t)))$  \tcp*[R]{Propagate Nominal.}

    $([t_{k_i,i}, \infty), p_i^{\can, t_s}, u_i^{\can, t_s}) \leftarrow$ \textsc{ConstructCandidate}$(p_i^\nom, u_i^\nom, \Pcal_\com)$ \tcp*[R]{Try to construct candidate.}

    \eIf{$([t_{k_i,i}, \infty), p_i^{\can, t_s}, u_i^{\can, t_s}) \neq $ null }{
        $([t_{k_i,i}, \infty), p_{k_i,i}^\com, u_{k_i,i}^\com) \leftarrow ([t_{k_i,i}, \infty), p_i^{\can, t_s}, u_i^{\can, t_s})$ \tcp*[R]{Commit to valid candidate.}
    }{
        $([t_{k_i,i}, \infty), p_{k_i,i}^\com, u_{k_i,i}^\com) \leftarrow ([t_{k_i,i}, \infty), p_{k_i-1,i}^\com, u_{k_i-1,i}^\com)$ \tcp*[R]{Failure: continue previous.}
    }
\end{algorithm}

\begin{assumption}[Sequential Consistency]\label{assumption:no_simultaneous_iterations}
    We assume the planning iterations of agents are non-overlapping. That is, for agents $i, j$, the planning iterations $k_i, k_j$ occur at distinct times $t_{k_i} \neq t_{k_j}$ for all $i \neq j$. Furthermore, we assume the computation time $\Delta t_i$ is negligible with respect to the system dynamics, such that the state of the system remains effectively constant during planning iterations.
\end{assumption}
\begin{remark}
    This assumption guarantees sequential consistency, preventing race conditions where multiple agents simultaneously attempt to update their committed trajectories based on outdated information. 
    In practice, this is enforced by a Time-Division Multiple Access (TDMA) communication scheme, where each agent is assigned a unique time slot to perform planning iterations if necessary.
    This is also a common assumption in the literature, e.g. \cite{tordesillasMADERTrajectoryPlanner2022}.
\end{remark}

We will now define the notions of validity and a committed trajectory, then the multi-agent \gatekeeper{} algorithm. 
\begin{definition}[Valid]\label{def:valid}
    The $k$-th candidate trajectory of agent $i$, denoted $([t_{k_i,i}, \infty),p_{k_i,i}^\can, u_{k_i,i}^\can)$, is considered \emph{valid} if,

    \noindent 1) it remains in the safe set,
    \eqn{
        p_{k_i, i}^\can(t) \in \Scal, \forall t \geq t_{k_i},
    }
    2) reaches a backup set $\Ccal_{k_i,i}$ at some finite time $t_{k_i} + t_s + T_B$, and remains there for all future time,
    \eqn{
        p_{k_i,i}^\can (t) \in \Ccal_{k_i,i}, \forall t \geq t_s + T_B,
    }
    3) and is collision free with respect to the committed trajectories of all other agents,
    \eqn{
        \norm{p_{k_i,i}^\can(t) - p_{k_j,j}^\com(t)} \geq \delta, \forall j \neq i, \forall t \geq t_{k_i}.
    }
\end{definition}

\begin{definition}[Committed Trajectory]\label{def:committed_trajectory}
    At the $k$th iteration of agent $i$, let the agent construct some goal-oriented nominal trajectory $[t_{k_i,i}, t_{k_i,i} + T_H], (p_{k_i,i}^\nom, u_{k_i,i}^\nom)$.
    Define the set of all valid switch times,
    \eqn{
        \mathcal{I}_{k_i,i} = \{t_S \in [0, T_H] \mid ([t_{k_i,i}, \infty), p_{k_i,i}^{\can, t_s}, u_{k_i,i}^{\can, t_s}) \text{ is valid}\}, 
    }
    where $([t_{k_i,i}, \infty), p_{k_i,i}^{\can, t_s}, u_{k_i,i}^{\can, t_s})$ is the candidate trajectory with switch time $t_s$ as in \cref{def:candidate_trajectory}, and validity is determined by \cref{def:valid}.
    If $\mathcal{I}_{k_i,i} \neq \emptyset$, let $t_s = \max (\mathcal{I}_{k_i,i})$.
    The agent then commits the candidate trajectory with switch time $t_s$,
    \eqn{
        ([t_{k_i,i}, \infty), p_{k_i,i}^\com, u_{k_i,i}^\com) = ([t_{k_i,i}, \infty), p_{k_i,i}^{\can, t_s}, u_{k_i,i}^{\can, t_s}),
    } 
    and thus must execute this trajectory until a new trajectory is committed.
    If $\mathcal{I}_{k_i,i} = \emptyset$, the agent continues executing its previously committed trajectory, $([t_{k_i,i}, \infty), p_{k_i,i}^\com, u_{k_i,i}^\com) = ([t_{k_i-1,i}, \infty), p_{k_i-1,i}^\com, u_{k_i-1,i}^\com)$.
\end{definition}

\subsection{Formation Control With Multi-Agent \gatekeeper{}}\label{sec:subsection:formation_control_with_multi_agent_gatekeeper}
We now discuss the specific application of multi-agent \gatekeeper{} to the leader-follower formation control \cref{problem:multi_agent_formation_control}.
Our key insight is to leverage the pre-computed leader trajectory as a \emph{shared, guaranteed-safe backup maneuver} for all followers. 
Since the trajectory is known to be safe and dynamically feasible, any follower that merges onto this path and \emph{maintains appropriate separation from other agents} is guaranteed to remain safe for all future time.

\begin{definition}[Trajectory Backup Set]\label{def:trajectory_backup_set}
    Let $([t_0, t_f], p_L, u_L)$ be a safe, dynamically feasible leader trajectory.
    For any state $x$ coinciding with the leader's path at parameter $t_L \in [t_0, t_f]$ such that $x = p_L(t_L)$, and any current time $t_\text{now}$, we define the \emph{trajectory backup controller} as,
    \eqn{
        \pi^\back(x, t) = u_L(t_L + (t - t_\text{now})),
    }
    Executing this controller ensures the agent tracks the leader's path $p_L$ starting from $t_L$, effectively rendering the leader's trajectory a backup set.
\end{definition}

However, due to the risk of inter-agent collisions, we must still verify that the trajectory backup set is safe with respect to other agents before it can be committed. 

This requires us to pose the following requirements on the leader's trajectory.
\begin{definition}[Valid Trajectory Backup Set]\label{def:valid_trajectory_backup_set}
    A trajectory $([t_0, t_f], p_L(t), u_L(t))$ is a \emph{valid trajectory backup set} if,
    \noindent
    \textbf{1)} It is safe with respect to static obstacles,
    \eqn{
        p_L(t) \in \Scal, \forall t \in [t_0, t_f],
    }
    \textbf{2)} It is dynamically feasible, and
    \textbf{3)} there exists a curvature margin $\epsilon \in \Rnonneg$ such that for all path parameters $t_1, t_2 \in [t_0, t_f]$ and future offsets $\tau \geq 0$,
    \eqn{
        d(t_1, t_2) \geq \delta + \epsilon \implies \norm{r_L(t_1 + \tau) - r_L(t_2 + \tau)} \geq \delta \label{eq:epsilon_condition}
    }
    where $d(\cdot, \cdot)$ is the arc length along the leader's trajectory as in \cref{def:arc_length}.
\end{definition}
The condition \cref{eq:epsilon_condition} accounts for the curvature of the leader's trajectory. It ensures that an arc length separation of $\delta + \epsilon$ guarantees at least a safe $\delta$ of Euclidian distance for all future time.
In practice, a conservative bound on $\epsilon$ can be computed by sampling the leader path and computing pairwise Euclidean distances over bounded lookahead windows.

\begin{lemma}[Validity of a Trajectory Backup Set]\label{lemma:valid_trajectory_backup_set}
    Let the leader trajectory $([t_0, t_f], p_L(t), u_L(t))$ be a valid trajectory backup set by \cref{def:valid_trajectory_backup_set}. 
    Suppose agents $i$ and $j$ begin to execute $\pi^\back$ from the leader path at times $t_i$ and $t_j$ respectively, at path parameters $t_{Li}$ and $t_{Lj}$ such that,
    \eqn{
        p_i(t_i) = p_L(t_{Li}), \quad p_j(t_j) = p_L(t_{Lj}).
    }

    If the arc-length separation between the agents at time $t_c = \max(t_i, t_j)$ satisfies,
    \eqn{
        d(t_{Li} + (t_c - t_i), t_{Lj} + (t_c - t_j)) \geq \delta + \epsilon \implies \|r_i(t) - r_j(t)\| \geq \delta, \; \forall t \geq t_c,
    }
    i.e. the agents remain safe with respect to each other for all future time.
\end{lemma}
\begin{proof}
    WLOG, assume $t_i \leq t_j$, thus $t_c = t_j$. At time $t_c$, agent $j$ has just joined the path at parameter $t_{Lj}$, and agent $i$ has been traveling on the path for $(t_j - t_i)$ seconds, reaching parameter $t_{Li} + (t_j - t_i)$. By assumption, the arc length distance between these two parameters satisfies the condition in \cref{def:valid_trajectory_backup_set}. Thus, for all $\tau \geq 0$,
    \eqn{
        \norm{r_L((t_{Li} + (t_j - t_i)) + \tau) - r_L(t_{Lj} + \tau)} \geq \delta.
    }
    Substituting $\tau = t - t_j$ yields $\norm{r_i(t) - r_j(t)} \geq \delta, \forall t \geq t_j$.
\end{proof}

Therefore, by ensuring that the trajectory backup set satisfies \cref{def:valid_trajectory_backup_set}, we can guarantee that if agents joining the leader's trajectory are spaced sufficiently far apart (satisfy the conditions of \cref{lemma:valid_trajectory_backup_set}), they will remain safe for all future time.

\subsection{Backup Trajectory Generation and Coordination}\label{sec:subsection:backup_trajectory_generation_and_coordination}
While \cref{lemma:valid_trajectory_backup_set} provides the theoretical condition for safety, it does not specify how agents distributively coordinate to claim "slots" on the leader's path, nor how they generate the transition trajectories to reach them. 

\begin{figure}
    \begin{minipage}[t]{0.48\textwidth}
        \begin{algorithm}[H]
            \caption{Construct Candidate}
            \label{alg:construct_candidate}
            \SetKwComment{Comment}{// }{}
            \DontPrintSemicolon
            
            $\mathcal{I}_{valid} \leftarrow \emptyset$\;
            
            \For{$t_s \in [t_k, t_k + T_H]$}{
                $x_s \leftarrow p_i^{nom}(t_s)$\;
                
                \Comment{Plan rejoin leader maneuver}
                $([t_{k_i,i} + t_s, t_{k_i,i} + t_s + T_B], p_i^{join}, u_i^{join}), t_{Li} \leftarrow$ \textsc{PlanJoinToBackup}$(x_s, p_L)$\;
                
                \If{$t_{Li} = $ null}{
                    \textbf{continue}\;
                }
                
                $([t_{k_i,i}, \infty), p_i^{can}, u_i^{can}) \leftarrow (p_i^{nom}|_{[t_k, t_s]}, p_i^{join}, p_L |_{[t_{Li}, \infty)})$\;
                
                \If{\textsc{ValidateCandidate}$(p_i^{can}, u_i^{can}, \mathcal{P}_{com}, t_{Li})$}{
                    $\mathcal{I}_{valid} \leftarrow \mathcal{I}_{valid} \cup \{t_s\}$\;
                }
            }
            
            \eIf{$\mathcal{I}_{valid} \neq \emptyset$}{
                $t_s^* \leftarrow \max(\mathcal{I}_{valid})$\;
                \KwRet candidate with switch time $t_s^*$\;
            }{
                \KwRet null\;
            }
        \end{algorithm}
    \end{minipage}
    \hfill
    \begin{minipage}[t]{0.48\textwidth}
        \begin{algorithm}[H]
            \caption{Validate Candidate}
            \label{alg:validate_candidate}
            \SetKwComment{Comment}{// }{}
            \DontPrintSemicolon
            
            \Comment{Check static obstacle collisions}
            \For{$t \in [t_k, t_{Li}]$}{
                \If{$p_i^{can}(t) \notin \mathcal{S}$}{
                    \KwRet \textbf{false}\;
                }
            }
            
            \Comment{Validate against all other agents' committed trajectories}
            \ForAll{$p_j^{com} \in \mathcal{P}_{com}$}{
                \Comment{Check inter-agent collisions}
                $t_{max} \leftarrow \max(t_{Li}, t_{Lj})$
                \For{$t \in [t_k, t_{max})$}{
                    \If{$\|r_i^{can}(t) - r_j^{com}(t)\| < \delta$}{
                        \KwRet \textbf{false}\;
                    }
                }
                \Comment{Ensure backup slot separation (Lemma 1)}
                \If{$\|r_i^{can}(t_{max}) - r_j^{com}(t_{max})\| < \delta + \epsilon$}{
                    \KwRet \textbf{false}\;
                }
            }
            \KwRet \textbf{true}\;
        \end{algorithm}
    \end{minipage}
\end{figure}

\subsubsection{Algorithm 2: Construct Candidate}
The goal of \cref{alg:construct_candidate} is to identify a valid switch time $t_s$ that allows the agent to transition from its nominal trajectory to the leader's trajectory, while maintaining safety.
The algorithm iterates through discrete times steps within the planning horizon (Line $2$).
At each potential switch time $t_s$, the agent attempts to compute a feasible "join" trajectory $(p_i^{join}, u_i^{join})$ that connects the agent's nominal $x_i^\nom(t_s)$ to some point on the leader's trajectory $p_L$ (Line $4$). 
If a feasible trajectory is found, the agent identifies the merge parameter $t_{Li}$ (point on the leader's path where the join occurrs).
It then constructs a full candidate (Line $7$) as Nominal -> Join -> Leader Path, and validates it agains static obstacles and other agents' committed trajectories using \cref{alg:validate_candidate} (Line $8$).
Finally, if any valid candidates were found, the candidate with the maximum switch time is returned (Lines $11$-$15$).

\subsubsection{Algorithm 3: Validate Candidate}
\Cref{alg:validate_candidate} provides the computational procedure for validating a candidate trajectory (\cref{def:valid}). Three conditions must be verified,
\begin{enumerate}
    \item \emph{Static Safety}: The candidate must remain in the safe set (Lines $1$-$3$).
    \item \emph{Transition Safety}: The candidate must avoid collisions with all other agents' committed trajectories during the transition period (Lines $4$-$7$).
    \item \emph{Backup Slot Safety}: Once the agent merges with the leader's trajectory, it must satisfy the separation conditions of \cref{lemma:valid_trajectory_backup_set} with respect to all other agents (Lines $8$-$9$).
\end{enumerate}
Note that line $8$ uses Euclidean distance as a conservative approximation of the arc length condition. Since the arc length is always greater than or equal to the Euclidean distance, ($d(a, b) \geq \|r_L(a) - r_L(b)\|$), satisfying the Euclidean condition is a computationally simpler sufficient condition for \cref{lemma:valid_trajectory_backup_set}

\subsubsection{Implementation Details}

In order to guarantee safety, we require that every agent is able to construct a valid committed trajectory at the initial time $t_0$. 
Each agent should construct their initial committed trajectory sequentially, with respect to the committed trajectories of all previous agents. If an agent is unable to construct a valid committed trajectory at time $t_0$, the algorithm is infeasible.

Note the use of euclidean distance in \cref{alg:validate_candidate} to verify sufficient separation along the leader's trajectory at the join time $t_{Li}$. 
This is a conservative approximation of the arc length condition in \cref{def:valid_trajectory_backup_set}, but is simpler to compute and suffices in practice. The arc length separation condition implies Euclidian separation if the maximum curvature of the path is bounded.

The function \textsc{PlanJoinToBackup}$(x_s, p_L)$ computes a dynamically feasible trajectory from state $x_s$ to some point on the leader's trajectory $p_L$, returning both the join trajectory and the time $t_{Li}$ at which the agent will reach the leader's trajectory. If no such trajectory can be found, it returns null.
In practice, we found that selecting multiple points along the leader's trajectory as potential targets for the join maneuver, and attempting to plan a trajectory to each point in sequence until one is found, works well.

We will now provide a formal proof of safety for the multi-agent formation control \cref{problem:multi_agent_formation_control}.
\begin{theorem}\label{theorem:safety} 
    Consider a set of $N$ agents satisfying \cref{assumption:communications} and \cref{assumption:no_simultaneous_iterations}. If
    \begin{enumerate}
        \item The leader trajectory $([t_0, t_f], p_L, u_L)$ satisfies \cref{def:valid_trajectory_backup_set}, with known $\epsilon \in \Rnonneg$ satisfying \cref{lemma:valid_trajectory_backup_set},
        \item All agents possess valid committed trajectories at time $t_0$, and update their committed trajectories according to multi-agent \gatekeeper{},
    \end{enumerate}
    Then, $\forall t \geq t_0$, all agents $i$ remain in $\Scal$, and maintain inter-agent safety $\norm{r_i(t) - r_j(t)} \geq \delta, \forall i \neq j$.
\end{theorem}
\begin{proof}
    The proof is by induction.

    \textbf{Base Case:} At time $t_0$, by assumption, each agent $i$ has a valid committed trajectory $([t_{0}, \infty), p_{0,i}^\com, u_{0,i}^\com)$, and safety holds by hypothesis. 

    \textbf{Inductive Step:} 
    Assume agent $i$ is safe at iteration $k_i$. At iteration $k_i + 1$, it generates a candidate.
    \begin{enumerate}
        \item \emph{Case A:} The candidate is valid by \cref{def:valid}. By \cref{def:valid}, the candidate is safe with respect to static obstacles, reaches the backup set, and is collision free with respect to all other agents' committed trajectories.
        \item \emph{Case B:} The agent finds no valid candidate. It continues executing its previous committed trajectory, valid by the inductive hypothesis, and thus remains safe.
    \end{enumerate}
    
    Since the algorithm guarantees agents never switch to an invalid candidate, and the initial committed trajectory is valid, safety is preserved at each iteration. 
\end{proof}
\begin{remark}
    The safety guarantees in \cref{theorem:safety} rely on assumptions which require timely, and asynchronous sharing of committed trajectories among agents, and the existence of initial valid committed trajectories constructed sequentially at $t_0$.
    In practice, communications are subject to delays and packet drops, and initial construction may fail for large numbers of agents.
    We therefore treat the present results as a proof of concept that the trajectory backup set idea yields strong safety guarantees under these conditions, and addressing delayed/partial communications and decentralized initial allocation is important future work \cref{section:conclusion}.
\end{remark}

\section{Simulation Results}\label{section:simulation_results}
We simulate a team of agents modeled as 3D Dubins vehicles executing \cref{problem:multi_agent_formation_control} in a dense urban-like 3D environment. 
We demonstrate the efficacy of our approach through a series of randomized trials, comparing multi-agent \gatekeeper{} to baseline methods.
We consider 3D curvature- and pitch-angle constrained vehicles, which are useful models for fixed-wing aircraft, and reflect the dynamics of Dubins Airplanes \cite{valavanisHandbookUnmannedAerial2015,dubinsCurvesMinimalLength1957}.

\subsection{Setup}
Consider a team of $N_A$ agents, modeled as a 3D variant of the common unicycle model, with pitch, curvature and velocity constraints. This model is useful due to its compatibility with 3D Dubins paths.
\begin{equation}
\begin{aligned}
    s_i &= \begin{bmatrix}
        x\\
        y\\
        z\\
        \psi \\
        
    \end{bmatrix}, \;
    \dot s_i &= \begin{bmatrix}
         v_\text{max} \cos \psi \cos \gamma\\
         v_\text{max} \sin \psi \cos \gamma\\
         v_\text{max} \sin \gamma \\
         \omega\\
    \end{bmatrix},\;
    u = \begin{bmatrix}
        \omega \\ \gamma
    \end{bmatrix}, \;
    \begin{aligned}
        |\omega| &\leq \omega_{\text{max}}\\
        \gamma_{\min} &\leq \gamma \leq \gamma_{\max}
    \end{aligned}
\end{aligned}
\end{equation}


Using this model, we can compute the leader's path using a Dubins-based \rrts{} method \cite{linPathPlanningUsing2014}, specifying the turn radius for the Dubins primitives accordingly. We used \cite{vanaMinimal3DDubins2020} to generate Dubins curves in 3D which respect these pitch and curvature constraints. 

Nominal follower paths are offset by a fixed displacement vector $d_i^*$, transformed to the body frame of the leader aircraft at each point along the path. The resulting paths may not be dynamically feasible nor are they guaranteed to be safe.
\Cref{fig:3D_dubins_path} shows a leader trajectory generated using this method, and the nominal follower trajectories generated by offsetting the leader trajectory. 
We see that while the leader trajectory avoids obstacles and respects dynamics constraints, nominal follower trajectories intersect with obstacles and violate minimum curvature constraints impose by the dynamics model.
This method is scalable to many agents, as once the leader trajectory is generated offline, additional followers can be added online.

\begin{figure}
    \centering
    \includegraphics[width=0.7\linewidth]{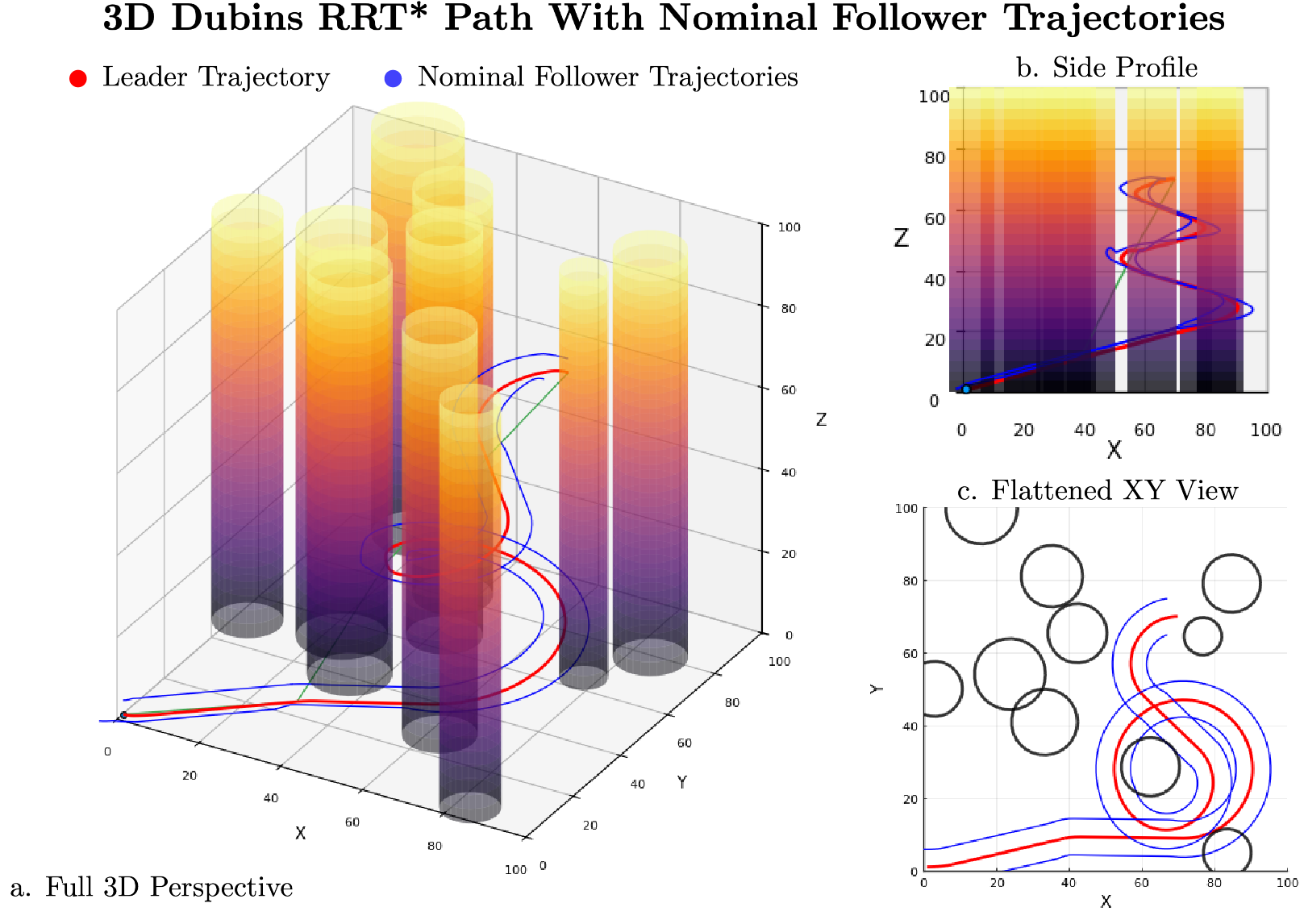}
    \caption{A leader trajectory (red) generated using Dubins RRT* in a 3D environment with cylindrical obstacles. (a) Full 3D perspective. (b) Side profile, demonstrating the path's adherence to pitch constraints. (c) Top-down XY projection. While the leader's path is safe and dynamically feasible, the nominal follower trajectories (blue) , created from a fixed offset , are shown intersecting with obstacles and violating dynamic constraints.}
    \label{fig:3D_dubins_path}
\end{figure}

\subsection{Simulation Results}
We perform simulations of one leader agent and two follower agents operating in randomly-generated environments. All simulations were performed in \texttt{julia}, running on a 2022 MacBook Air (Apple M2, 16GB).

The simulation environment is $100\times100\times100$ meters and contains $25$ randomly placed cylindrical obstacles with varying radii. The leader trajectory is generated from $(0, 0, 0)$ to $(100, 100, 70)$, with the start and end orientation fixed as level and towards positive $x$. Two follower agents are offset by $(-3.0, 5.0, 0.0)$ and $(-3.0, -5.0, 0.0)$ relative to the body frame of the leader. A minimum turning radius of $10$ and pitch-limits of $(-15^\circ, 20^\circ)$ are used.  An inter-agent collision distance $\delta = 1$ is used.
Each environment is generated by placing each cylinder at a random location, and sampling the radius uniformly between $2$ and $5$ units. The leader trajectory is then generated using the Dubins \rrts{} method described above.

\Cref{fig:3d_gk_solution_times} shows the evolution of a sample simulation over time. We see that at various points in time, agents rejoin the leader's path to avoid collision with obstacles, while still maintaining a safe distance. \Cref{fig:3d_gk_solution} provides a more detailed view of the same scenario, showing the full 3D trajectory, a 2D projection of the scenario, as well as the minimum interagent distance plotted over the course of the trial. We see that agents deviate as necessary to avoid obstacles while remaining safe with respect to other agents. Agents must also deviate to make tight turns that are dynamically infeasible, but overall remain close to the nominal trajectory

\begin{figure}
    \centering
    \includegraphics[width=0.8\linewidth]{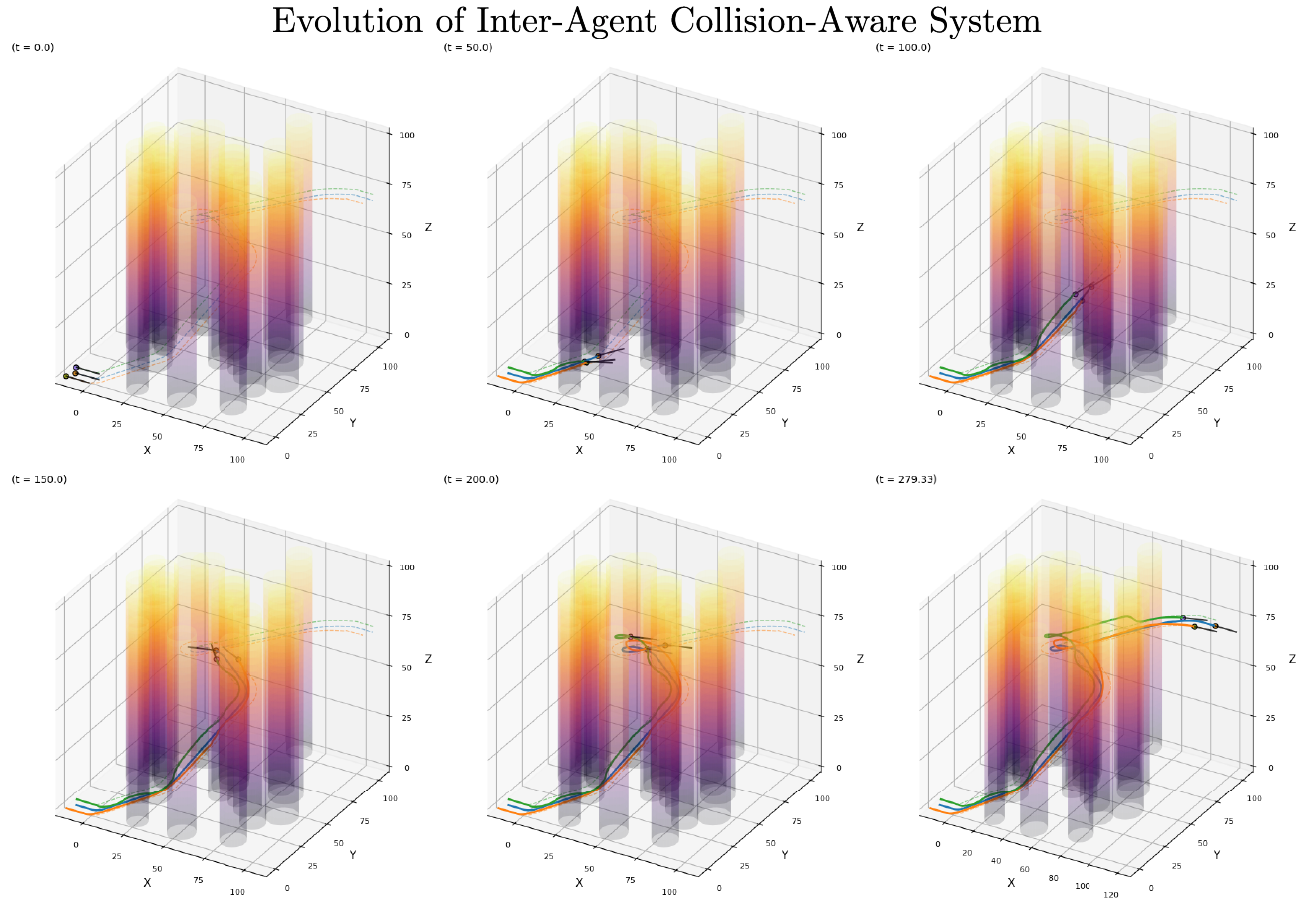}
    \caption{Evolution of a sample simulation over time, shown at six distinct time steps. The agents navigate a dense 3D environment with cylindrical obstacles. Agents can be seen deviating from their nominal formation and rejoining the leader's path to safely avoid collisions with obstacles.}
    \label{fig:3d_gk_solution_times}
\end{figure}

\begin{figure}[h!]
    \centering
    \includegraphics[width=0.95\linewidth]{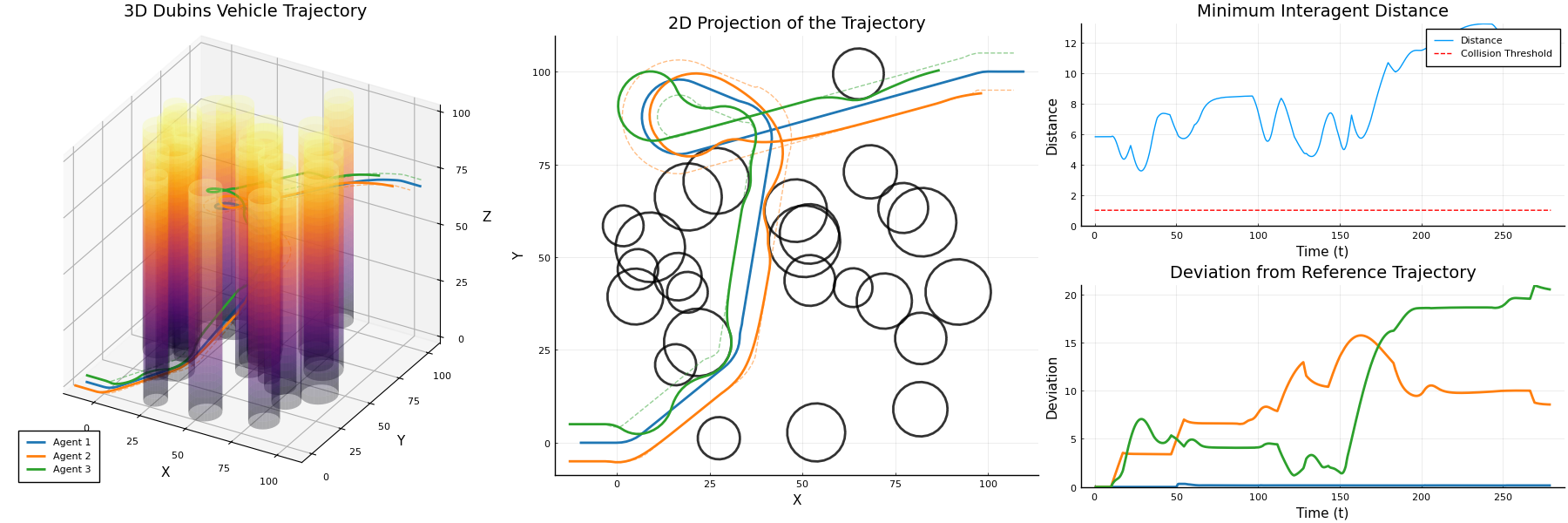}
    \caption{Trajectory and performance metrics for the simulation scenario shown in \cref{fig:3d_gk_solution_times}. (Left) Full 3D agent paths; (Center) 2D projection with static obstacles; (Top-Right) minimum inter-agent distance over time, which remains above the collision threshold; and (Bottom-Right) agent deviation from their nominal reference trajectories.}
    \label{fig:3d_gk_solution}
\end{figure}

\subsection{Baseline Comparison}

To validate our approach, we compare multi-agent \gatekeeper{} with baseline methods (CBF-QP and NMPC) by simulating each algorithm in 100 randomly-generated scenarios. We allow the baseline methods to replan at a frequency of $60$ Hz, as well as vary velocity between $[0.8, 1.0]$ to improve their ability to avoid collisions. The other simulation parameters remain unchanged. 
To ensure feasibility of the optimization-based solvers, in the event that the optimization problem becomes infeasible, we relax the inter-agent collision constraints by moving them to the cost function with a high penalty weight.
We use a formation tracking error metric defined as the average Euclidean distance between each follower and its nominal position relative to the leader, averaged over the duration of the trial. Success is defined as no agents incurring a collision over the duration of the trial. The results of the trials are given in \cref{tab:baseline_comparison}. We see that multi-agent \gatekeeper{} is able to successfully complete all $100$ trials without collision, while maintaining low formation tracking error.
The NMPC baseline achieves a moderate success rate, and with less formation error than multi-agent \gatekeeper{}. This highlights the tradeoff between safety and performance, as our approach prioritizes safety over formation tracking performance, and thus must deviate more frequently from the nominal trajectory.
The CBF-QP method struggled significantly with the highly nonlinear and dynamic collision avoidance constraints, resulting in very poor success rates, even with a high replanning frequency and planning horizon, as well as velocity modulation, as shown in \cref{tab:baseline_comparison}.  
Our method was the \textbf{only} method to achieve a 100\% success rate, demonstrating our superior safety guarantees. 

\begin{table}[h!]
    \centering
    \begin{tabular}{|ccc|}
        \hline 
        Algorithm & Success Rate &  Mean Formation Error \\
        \hline
        multi-agent \gatekeeper{} & \textbf{100\%} & 6.32\\
        CBF-QP & 22\% & 7.48\\
        NMPC & 68\% & \textbf{4.32}\\
        \hline
    \end{tabular}
    \caption{
        Restuls of the baseline comparison over 100 randomized trials.
        Each scenario was created by randomly generating a 3D environement with obstacles, and generating a leader trajectory using Dubins RRT*. The same leader trajectory was used across all algorithms for each scenario.
        Success Rate is the percentage of trials completed with no collisions. Mean Formation Error is the average deviation from the nominal reference trajectory.}
    \label{tab:baseline_comparison}
\end{table}
\section{Hardware Demonstration}\label{section:hardware_experiments}
To validate the physical feasibility of the complex, 3D trajectories generated by the multi-agent \gatekeeper{} algorithm, we executed a representative scenario on a team of Crazyflie 2.0 quadcopters.
\vspace{-0.5cm}
\subsection{Experiment Design}
The experiment's goal was to confirm that the safe trajectories, which require agents to dynamically break and reform their formation, satisfies safety constraints.
Due to the challenges of onboard, high-frequency replanning and information sharing onboard the quadcopters, the trajectories were pre-computed offline using the full multi-agent \gatekeeper{} algorithm in a simulated environment.
These trajectories were then used as a reference for the onboard low-level controller, which used state feedback from a Vicon motion capture system to track the path.

The quadcopters were controlled to emulate the 3D Dubins vehicle model used in simulation.
This was achieved by commanding the low-level controller to track the pre-computed path's position, velocity, and yaw, thereby respecting the pitch and curvature constraints of the simulated Dubins airplanes.

\subsection{Results}
We ran a scenario where three agents, starting in a triangular formation, were required to navigate through a narrow gate. As shown in \cref{fig:hardware_one}, the only safe solution requires the follower agents to identify the future collision risk, break their formation and follow the leader's path (their 'trajectory backup set') in a single-file line, and then re-establish the formation after clearing the obstacle.

This experiment was performed five times. In all trials, the quadcopters successfully executed the maneuver, passed through the gate without collision, and maintained the required safety separation.
This result validates that the trajectories produced by our algorithm are not just theoretically sound but are also dynamically feasible and trackable by physical robotic systems in a 3D environment.

\begin{figure}
    \centering
    \includegraphics[width=0.9\linewidth]{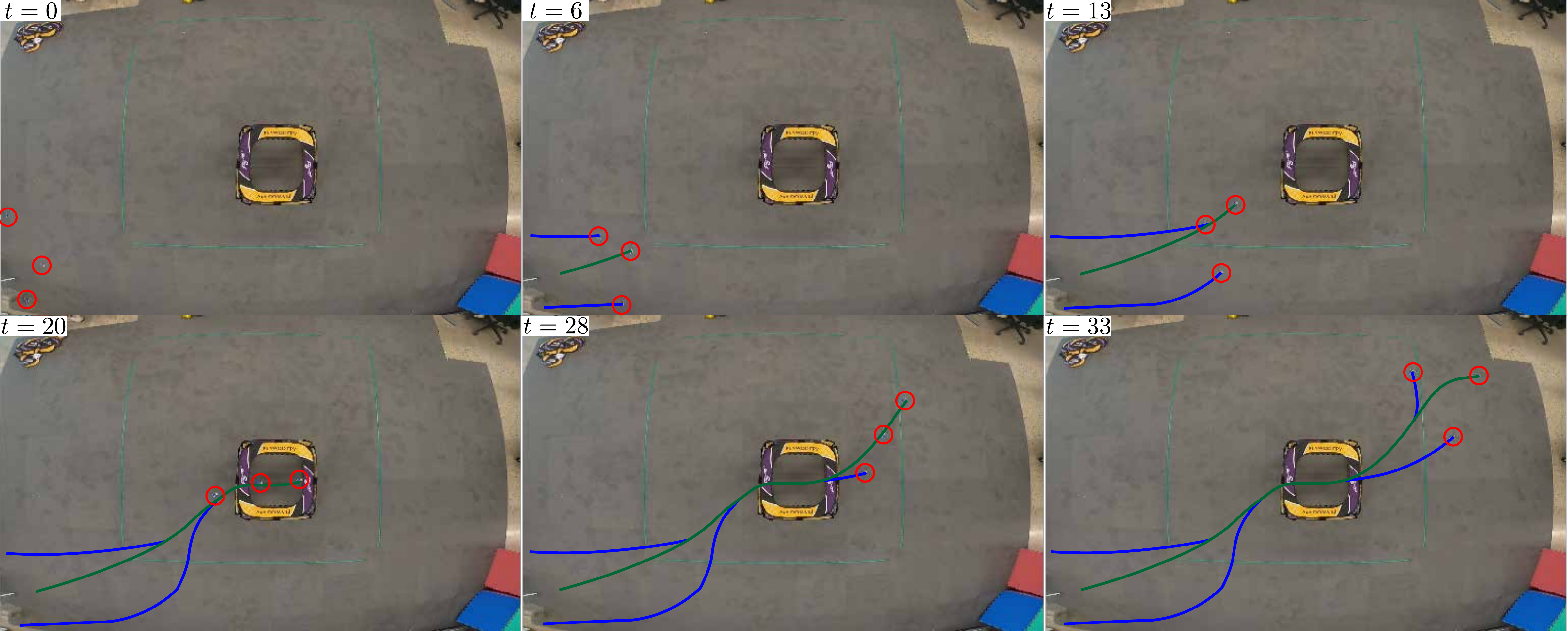}
    \caption{Hardware demonstration of multi-agent \gatekeeper{} with three Crazyflie 2.0 quadcopters navigating through a narrow gate.}
    \label{fig:hardware_one}
\end{figure}

\section{Conclusion}\label{section:conclusion}
This paper introduced multi-agent \gatekeeper{}, a novel, distributed safety framework for leader-follower formation control in complex 3D environments. Our approach extends the single-agent \gatekeeper{} safety verification algorithm to a multi-agent system, and marks its first application in a 3D environment.

The core of our method is a hybrid online-offline approach, where a leader follows a pre-computed safe path, and follower agents compute trajectories online. 
We introduced the "trajectory backup set," which leverages the leader's safe path as a guaranteed backup maneuver for all followers, allowing followers to flexibly maintain formation, but maintain a safe fallback option.
We then provided a rigorous proof that this approach ensures forward invariant safety under the assumption of ideal, delay-free communication and sequential decision making.

Empirically, we demonstrated the superior reliability of multi-agent \gatekeeper{} compared to baseline methods, maintaining 100\% safety complex 3D environments where CBF and NMPC methods frequently failed. Furthermore, we validated the approach on hardware with a team of quadcopters, successfully demonstrating the ability to break formation to pass through a narrow obstacle and reform afterward while still ensuring safety.

Limitations of the current work include the offline computation of trajectories for the hardware demonstration, the assumption of ideal, delay-free communications, and the centralized nature of the information sharing. Future work will investigate scalability and decentralization of this approach, and its robustness to communication delays and bandwidth constraints.
\bibliography{biblio}

\end{document}